\documentclass[11pt]{article}

\usepackage{amsfonts, amsmath,amssymb,array,latexsym, amscd}
\usepackage[hidelinks,colorlinks=true,linkcolor=blue,citecolor=blue]{hyperref}

\usepackage{fancyhdr,a4wide}
\usepackage{amsthm}
\usepackage{mathtools} 
\usepackage{subcaption}
\usepackage{xcolor}
\usepackage{amsmath}
\usepackage{bm}
\usepackage{thm-restate}

\newtheorem{theorem}{Theorem}[section]
\newtheorem{lemma}[theorem]{Lemma}

\newtheorem{proposition}[theorem]{Proposition}

\numberwithin{equation}{section}

\newcommand{\twiddle}[1]{\widetilde{#1}}

\newcommand{\norm}[1]{\left|\left|#1\right|\right|}
\newcommand{\lr}[1]{\left(#1\right)}
\newcommand{\abs}[1]{\left|#1\right|}
\newcommand{\set}[1]{\left\{#1\right\}}
\newcommand{\E}[1]{\mathbb E\left[#1\right]}

\newcommand{\R}{\mathbb R}

\newcommand{\gives}{\rightarrow}

\title{Depth Dependence of $\mu$P Learning Rates in ReLU MLPs}

\author{Samy Jelassi$^{1}$\footnote{Work done while interning at Google NYC.}, Boris Hanin$^{1}$, Ziwei Ji$^{2}$, Sashank J.\ Reddi$^{2}$,\\
Srinadh Bhojanapalli$^{2}$, Sanjiv Kumar$^{2}$}
\date{%
    $^1$Princeton University\\
    $^2$Google Research, NYC
}
\begin{document}

\maketitle

\begin{abstract}
  In this short note we consider random fully connected ReLU networks of width $n$ and depth $L$ equipped with a mean-field weight initialization. Our purpose is to study the dependence on $n$ and $L$ of the  maximal update ($\mu$P) learning rate, the largest learning rate for which the mean squared change in pre-activations after one step of gradient descent remains uniformly bounded at large $n,L$. As in prior work on $\mu$P \cite{yang2022tensor}, we find that this maximal update learning rate is independent of $n$ for all but the first and last layer weights. However, we find that it has a non-trivial dependence of $L$, scaling like $L^{-3/2}.$ 
\end{abstract}

\section{Introduction}\label{sec:intro}
Using a neural network requires many choices. Even after fixing an architecture, one must still specify initialization scheme, learning rate (schedule), batch size, data augmentation, regularization strength, and so on. Moreover, model performance is often highly sensitive to the setting of these hyperparameters, and yet exhaustive grid search type approaches are computationally expensive. It is therefore important to develop theoretically grounded principles for reducing the cost of hyperparameter tuning. In this short note we focus specifically on the question of how to select learning rates in a principled way. More precisely, our purpose is to generalize the \textit{maximal update} ($\mu$P) approach of \cite{yang2022tensor} to setting learning rates to take into account network depth. 

\subsection{Overview of $\mu$P Approach to Learning Rates}
We study learning rates in the simple setting of depth $L$ fully connected neural networks with ReLU activations and a uniform value $n$ for the input dimension and the hidden layers widths.\footnote{Our computations readily generalize to the case of variable layer widths. Indeed, we carry out the proof of Theorem \ref{thm:main} in this context.} In such a network, by definition, each input $x\in \R^{n}$ produces an output $z^{(L+1)}(x)\in \R^n$ through a sequence of pre-activations $z^{(\ell)}(x)\in \R^n$ given by 
\begin{align}
\label{eq:z-def}    z^{(\ell+1)}(x)&=\begin{cases}W^{(\ell+1)}\sigma\lr{z^{(\ell)}(x)},&\quad \ell \geq 1\\
    W^{(1)}x,&\quad \ell = 0
    \end{cases},\qquad \sigma(t):=\max \set{0,t}.
\end{align}
Selecting learning rates cannot be done independently of an initialization scheme. As in \cite{yang2022tensor}, we draw random weights for the network \eqref{eq:z-def} from the so-called  mean-field initialization 
\begin{align}
 \label{eq:W-def}   W_{ij}^{(\ell)} \sim \begin{cases}\mathcal N\lr{0,2/n},&\quad \ell =1,\ldots, L\\
    \mathcal N\lr{0,1/n^2},&\quad \ell = L+1\end{cases}.
\end{align}
The factor of two in variance of hidden layer weights corresponds to the well-known He initialization \cite{he2015delving}, which ensures that the expected squared activations neither grow nor decay with depth:
\begin{equation}\label{eq:He-init}
\E{\norm{z^{{(\ell)}}(x)}^2} = \norm{x}^2,\quad \forall \ell =1, \ldots, L.    
\end{equation}
The much smaller variance of weights in the final layer distinguishes the initialization scheme \eqref{eq:W-def} from the so-called NTK initialization \cite{jacot2018neural}. The difference is twofold. First, when $n$ is large the network output $z^{(L+1)}(x)$ is close to zero. However, crucially, the parameter gradients $\nabla_\theta z^{(L+1)}(x)$ are remain non-zero. Second, even in the infinite width limit $n\gives \infty$ networks trained by gradient descent are capable of feature learning \cite{mei2018mean,nguyen2020rigorous,rotskoff2018parameters,yang2022tensor}. This is in contrast to the setting where the final layer weight variance scales like $1/n$, which corresponds to the kernel regime in which neural networks trained by SGD with a small learning rate on a mean squared error loss converge to linear models and hence cannot learn data-adaptive features \cite{du2018gradient,jacot2018neural,liu2022loss}. 

A key contribution of \cite{yang2022tensor} is that the initialization \eqref{eq:W-def} not only leads to feature learning at large $n$ but also allows for zero-shot learning rate transfer with respect to variable width. This means that, empirically, for a fixed depth $L$ the learning rate at small $n$ that leads to the smallest training loss after one epoch is close to constant as one varies $n$\footnote{Strictly speaking, the $\mu$P prescription gives $n$-dependent learning rates for weights in the first and last layer and $n$-independent learning rates for weights in other layers (see Table 3 of \cite{yang2022tensor}).}. Hence, in practice, one may do logarithmic grid search for good learning rates in relatively small models (with small $n$) and then simply re-use the best learning rate for wider networks. 

\subsection{Main Result: Extending the $\mu$P Heuristic to Deeper Networks} Instead of studying directly the training loss  after one epoch \cite{yang2022tensor} introduces what we will refer to here as the \textit{maximal update heuristic}, which says that a good learning rate is one that corresponds to the largest change in hidden layer pre-activations after one step of GD that does not lead to a divergence at large $n$. More precisely, the relation \eqref{eq:He-init} shows that $i$-th neuron pre-activation in layer $\ell$ corresponding to an input $x$ that satisfies
\[
\E{\lr{z_i^{(\ell)}(x)}^2} = \frac{1}{n}\norm{x}^2,\qquad i=1,\ldots, n,\quad \ell = 1,\ldots, L,
\]
with the average being over  initialization. To study the change in neuron pre-activations under GD we consider a batch $\mathcal B= \set{(x,y)}$ size of $1$ and the associated mean-squared error
\[
\mathcal L_{\mathcal B}(\theta):=\frac{1}{2} \norm{z^{(L+1)}(x;\theta)-y}^2,
\]
where we've emphasized the dependence of the network output $z^{(L+1)}(x;\theta)$ on the network weights $\theta$. Let us denote by 
\[
\Delta^{\mathcal B} z_i^{(\ell)}(x) = \text{change in }z_i^{(\ell)}(x) \text{ after first step of GD on }\mathcal L_{\mathcal B}.
\]
The maximal update heuristic then asks that we set the learning rate $\eta$ so that
\begin{equation}\label{eq:muP-heuristic}
\mu\text{P learning rate } \eta^*:= \text{learning rate for which } \E{\lr{\Delta^{\mathcal B} z_i^{(\ell)}(x)}^2} = 1,    
\end{equation}
where the average is over initialization. A priori, $\eta^*$ depends on both network $n$ width and depth $L$. The article \cite{yang2022tensor} shows that $\eta^*$ does not depend on $n$ and hence can be estimated accurately at small $n$. In this article, we take up the question of how $\eta^*$ depends on depth. The following theorem shows that $\eta^*$ is not depth-independent:
\begin{theorem}\label{thm:main}
For each $c_1>0$ there exists $c_2,c_3>0$ with the following property. Fix a network width $n$ and depth $L$ so that $L/n < c_1$. Then, 
\begin{equation}\label{eq:update-size}
    \sup_{n\geq 1}\abs{\E{\lr{\Delta^{\mathcal B} z_i^{(\ell)}(x)}^2} - c_2 \eta^2 \ell^3}\leq c_3\eta^2 \ell^2,
\end{equation}
where $\mathcal B=\set{(x,y)}$ is any batch of size one consisting of a normalized datapoint $(x,y)$ sampled independent of network weights and biases with:
\[
\E{\frac{1}{n}\norm{x}^2} = 1,\qquad \E{\norm{y}^2} = 1.
\]
\end{theorem}
\noindent Theorem \ref{thm:main} shows that the $\mu$P heuristic \eqref{eq:muP-heuristic} dictates that
\[
\eta^*(L) = \text{const}\cdot L^{-3/2}.
\]

\section{Proof of Theorem \ref{thm:main}}
\subsection{Notation and Problem Setting} We prove a slightly more general result than Theorem \ref{thm:main} in two senses. First, we allow for variable widths:
\[
n_\ell = \text{width of layer }\ell = 0,\ldots, L+1
\]
Second, we will also allow for parameter-dependent learning rates: 
\[
\eta_\mu = \text{ learning rate used for parameter }\mu.
\]
At the end we will restrict to the case where $\eta_\mu =\eta$ is independent of $\mu$. Moreover, in order to state our proof most efficiently, we introduce some notation. Namely, we will write $x_\alpha\in \R^{n_0}$ for the network input at which we study both the forward and backward pass and will denote for brevity
\[
z_{i;\alpha}^{(\ell)}:=z_i^{(\ell)}(x_\alpha),\qquad z_\alpha^{(\ell)}:=z^{(\ell)}(x_\alpha).
\]
Thus, the batch loss $\mathcal L_{\mathcal B}$ we consider is
\[
\frac{1}{2}\norm{z_\alpha^{(L+1)}-y_\alpha}^2.
\]
Further, we abbreviate
\[
\Delta z_{i;\alpha}^{(\ell)} := \Delta^{\mathcal B} z_{i;\alpha}^{(\ell)}.
\]
With this notation, the forward pass now takes the form
\begin{align*}
    z_{i;\alpha}^{(\ell+1)} = \begin{cases}\sum_{j=1}^{n_0}W_{ij}^{(1)}x_{j;\alpha} ,&\quad \ell = 0\\ 
    \sum_{j=1}^{n_{\ell-1}} W_{ij}^{(\ell)} \sigma\lr{z_{j;\alpha}^{(\ell)}},&\quad \ell = 1,\ldots, L\end{cases}
\end{align*}
and the initialization scheme is
\begin{align*}
    W_{ij}^{(\ell+1)}\sim \begin{cases}
    \mathcal N\lr{0,\frac{1}{n_L^2}},&\quad \ell = L\\
    \mathcal N\lr{0,\frac{2}{n_{\ell}}},&\quad \ell=0,\ldots, L-1
    \end{cases}.
\end{align*}

\subsection{Proof Details }
We begin with the following Lemma.

\begin{lemma}\label{lem:lemma_one}
  For any depth $\ell\leq L$, the pre-activation change satisfies
  \begin{align*}
     \mathbb{E}[(\Delta  z_{i;\alpha}^{(\ell)})^2]=A^{(\ell)}+B^{(\ell)},
  \end{align*}
  where

\vspace*{-.6cm}
  
\begin{align}
   A^{(\ell)}&:=\mathbb E \left[\frac{1}{n_L^2}\sum_{\mu_1,\mu_2\leq \ell}\eta_{\mu_1}\eta_{\mu_2} \partial_{\mu_1}z_{1;\alpha}^{(\ell)}\partial_{\mu_2}z_{1;\alpha}^{(\ell)}\right.\\
   &\qquad \qquad \times \left.\frac{1}{n_L^2}\sum_{j_1,j_2=1}^{n_L}\left\{\partial_{\mu_1}z_{j_1;\alpha}^{(L)}\partial_{\mu_2}z_{j_1;\alpha}^{(L)}\lr{z_{j_2;\alpha}^{(L)}}^2+2z_{j_1;\alpha}^{(L)}\partial_{\mu_1}z_{j_1;\alpha}^{(L)}z_{j_2;\alpha}^{(L)}\partial_{\mu_2}z_{j_2;\alpha}^{(L)}\right\}\right],\notag\\
B^{(\ell)}&:=\E{\frac{1}{n_L}\sum_{\mu_1,\mu_2\leq \ell}\eta_{\mu_1}\eta_{\mu_2} \partial_{\mu_1}z_{1;\alpha}^{(\ell)}\partial_{\mu_2}z_{1;\alpha}^{(\ell)} \frac{1}{n_L}\sum_{j=1}^{n_L} \partial_{\mu_1}z_{j;\alpha}^{(L)}\partial_{\mu_2}z_{j;\alpha}^{(L)}}.\label{label:Blinit}
\end{align}

\end{lemma}
\begin{proof}[Proof of \autoref{lem:lemma_one}]
   We first expand  $\Delta z_{i;\alpha}^{(\ell)}$ by applying the chain rule:
   \begin{align}\label{eq:chain_rule}
      \Delta z_{i;\alpha}^{(\ell)}= \sum_{\mu\leq \ell} \cdot\partial_\mu z_{i;\alpha}^{(\ell)}\Delta \mu,
   \end{align}
   where $\Delta \mu$ is the change in $\mu$ after one step of GD. The SGD update satisfies:
  \begin{align}\label{eq:sgd_proof}
     \Delta \mu = - \eta_\mu \partial_\mu \left\{\frac{1}{2}\norm{z_\alpha^{(L+1)}-y_\alpha}^2\right\} = -\eta_\mu \sum_{k=1}^{n_{L+1}} \partial_\mu z_{k;\alpha}^{(L+1)} \lr{z_{k;\alpha}^{(L+1)} - y_{k;\alpha}}.
  \end{align}
We now combine \eqref{eq:chain_rule} and \eqref{eq:sgd_proof} to obtain: 
\begin{align}\label{eq:deltaz_final_lem1}
   \Delta z_{i;\alpha}^{(\ell)}= \sum_{\mu\leq \ell}\sum_{k=1}^{n_{L+1}} \eta_\mu \partial_\mu z_{i;\alpha}^{(\ell)} \partial_\mu z_{k;\alpha}^{(L+1)} \lr{y_{k;\alpha}-z_{k;\alpha}^{(L+1)}}.
\end{align}
Using \eqref{eq:deltaz_final_lem1}, we obtain
\begin{align}
    \E{\lr{\Delta z_{i;\alpha}^{(\ell)}}^2}
    =& \E{\lr{\sum_{\mu \leq \ell}\eta_\mu \partial_\mu z_{1;\alpha}^{(\ell)}\partial_\mu z_{1;\alpha}^{(L+1)}\lr{z_{1;\alpha}^{(L+1)}-y_{1;\alpha}}}^2}\notag  \\
  =&\E{\sum_{\mu_1,\mu_2\leq \ell} \eta_{\mu_1}\eta_{\mu_2} \partial_{\mu_1}z_{1;\alpha}^{(\ell)}\partial_{\mu_2}z_{1;\alpha}^{(\ell)}\partial_{\mu_1}z_{1;\alpha}^{(L+1)}\partial_{\mu_2}z_{1;\alpha}^{(L+1)}\mathbb{E}_{y}\left[\lr{z_{1;\alpha}^{(L+1)}-y_{1;\alpha}}^2\right]}.  \label{E:2nd-moment-AB}  %&=\E{\sum_{\mu_1,\mu_2\leq \ell} \eta_{\mu_1}\eta_{\mu_2} \partial_{\mu_1}z_{1;\alpha}^{(\ell)}\partial_{\mu_2}z_{1;\alpha}^{(\ell)}\partial_{\mu_1}z_{1;\alpha}^{(L+1)}\partial_{\mu_2}z_{1;\alpha}^{(L+1)}\lr{\lr{z_{1;\alpha}^{(L+1)}}^2+1}}\\
%n \label{E:2nd-moment-AB}   &=: A^{(\ell)}+B^{(\ell)},
\end{align}
Given the distribution of $z_{1;\alpha}^{(L+1)}$ and $y$, we have 
\begin{align} \label{eq:expzy}\mathbb{E}_{y}\left[\lr{z_{1;\alpha}^{(L+1)}-y}^2\right] = (z_{1;\alpha}^{(L+1)})^2 + 1
\end{align}
We plug \eqref{eq:expzy} in \eqref{E:2nd-moment-AB}  and obtain 
 \begin{align}
     \mathbb{E}[(\Delta  z_{i;\alpha}^{(\ell)})^2]=A^{(\ell)}+B^{(\ell)},
  \end{align}
  where
\begin{align}
    \label{E:A-def} A^{(\ell)} &= \E{\sum_{\mu_1,\mu_2\leq \ell} \eta_{\mu_1}\eta_{\mu_2} \partial_{\mu_1}z_{1;\alpha}^{(\ell)}\partial_{\mu_2}z_{1;\alpha}^{(\ell)}\partial_{\mu_1}z_{1;\alpha}^{(L+1)}\partial_{\mu_2}z_{1;\alpha}^{(L+1)}\lr{z_{1;\alpha}^{(L+1)}}^2}\\
\label{E:B-def}    B^{(\ell)} &= \E{\sum_{\mu_1,\mu_2\leq \ell} \eta_{\mu_1}\eta_{\mu_2} \partial_{\mu_1}z_{1;\alpha}^{(\ell)}\partial_{\mu_2}z_{1;\alpha}^{(\ell)}\partial_{\mu_1}z_{1;\alpha}^{(L+1)}\partial_{\mu_2}z_{1;\alpha}^{(L+1)}}.
\end{align}
We integrate out the weights in layer $L+1$ in \eqref{E:A-def} and \eqref{E:B-def} which yields the stated result.

%We finally apply \textcolor{red}{Lemma ***} to integrate out the weights in layer $L+1$ in \eqref{E:A-def} and \eqref{E:B-def}. This yields the aimed result.

\end{proof}

\begin{lemma}\label{lem:lemma_two}
 For any depth $\ell\leq L$,  the constant $A^{(\ell)}$ in \autoref{lem:lemma_one} satisfies $A^{(\ell)}=O(n^{-1})$.
\end{lemma}
\begin{proof}[Proof of \autoref{lem:lemma_two}]
The result is obtained essentially the same analysis at we apply to $B^{(\ell)}$ below combined with the observation that there is an extra $1/n_L$ in front of $A^{(\ell)}$ compared with $B^{(\ell)}$.
\end{proof}

\autoref{lem:lemma_two} indicates that we may neglect the contribution of $A^{(\ell)}$ in \autoref{lem:lemma_one}. We now focus on obtaining a recursive description for $B^{(\ell)}$.

\begin{lemma}\label{lem:lemmatwoprime} For any depth $\ell\leq L$, the constant $B^{(\ell)}$ in \autoref{lem:lemma_one} satisfies
\begin{align}\label{eq:bl_simplification1}
   B^{(\ell)}=\E{\frac{1}{n_L}\sum_{\mu_1,\mu_2\leq \ell}\eta_{\mu_1}\eta_{\mu_2} \frac{1}{n_\ell^2}\sum_{j_1,j_2=1}^{n_\ell} \partial_{\mu_1}z_{j_1;\alpha}^{(\ell)}\partial_{\mu_2}z_{j_1;\alpha}^{(\ell)}\partial_{\mu_1}z_{j_2;\alpha}^{(\ell)}\partial_{\mu_2}z_{j_2;\alpha}^{(\ell)}}.
\end{align}

\end{lemma}

\begin{proof}[Proof of \autoref{lem:lemmatwoprime}]
The idea of this proof is to condition on $z_\alpha^{(\ell)}$ and integrate out weights in layers $\ell+1,\ldots, L$ to obtain
\begin{align}\label{eq:expcondlemmaA3}
  \E{  \frac{1}{n_L}\sum_{j=1}^{n_L} \partial_{\mu_1}z_{j;\alpha}^{(L)}\partial_{\mu_2}z_{j;\alpha}^{(L)}~\bigg|~z_\alpha^{(\ell)}} = \frac{1}{n_\ell}\sum_{j=1}^{n_\ell} \partial_{\mu_1}z_{j;\alpha}^{(\ell)}\partial_{\mu_2}z_{j;\alpha}^{(\ell)}.
\end{align}
This will yield the result once we plug \eqref{eq:expcondlemmaA3} into \eqref{label:Blinit}. To see \eqref{eq:expcondlemmaA3}, we proceed by induction on $L$ starting with $\ell = L$. In this case, the result is trivial. Suppose now $\ell < L$. Then we have
\begin{align*}
    &\E{  \frac{1}{n_L}\sum_{j=1}^{n_L} \partial_{\mu_1}z_{j;\alpha}^{(L)}\partial_{\mu_2}z_{j;\alpha}^{(L)}~\bigg|~z_\alpha^{(\ell)}}\\
    &\qquad =  \E{  \frac{1}{n_L}\sum_{j=1}^{n_L}\sum_{k_1,k_2=1}^{n_{L-1}} W_{jk_1}^{(L)}W_{jk_2}^{(L)}\partial_{\mu_1}\sigma\lr{z_{k_1;\alpha}^{(L-1)}}\partial_{\mu_2}\sigma\lr{z_{k_2;\alpha}^{(L-1)}} ~\bigg|~z_\alpha^{(\ell)}}\\
    &\qquad =  \E{  \frac{1}{n_L}\sum_{j=1}^{n_L}\frac{2}{n_{L-1}}\sum_{k=1}^{n_{L-1}} \partial_{\mu_1}\sigma\lr{z_{k;\alpha}^{(L-1)}}\partial_{\mu_2}\sigma\lr{z_{k;\alpha}^{(L-1)}} ~\bigg|~z_\alpha^{(\ell)}}\\
    &\qquad =  \E{ \frac{2}{n_{L-1}}\sum_{k=1}^{n_{L-1}} \lr{\sigma'\lr{z_{k;\alpha}^{(L-1)}}}^2 \partial_{\mu_1}z_{k;\alpha}^{(L-1)}\partial_{\mu_2}z_{k;\alpha}^{(L-1)} ~\bigg|~z_\alpha^{(\ell)}}\\
    &\qquad =\frac{1}{n_{L=1}}\sum_{k=1}^{n_{L-1}} \partial_{\mu_1}z_{k;\alpha}^{(L-1)}\partial_{\mu_2}z_{k;\alpha}^{(L-1)},
\end{align*}
where in the last equality we use that $\sigma'(z_{k;\alpha}^{(\ell)})$ is distributed according to a Bernoulli $1/2$ random variable and is independent of $\partial_{\mu_1}z_{k;\alpha}^{(L-1)}\partial_{\mu_2}z_{k;\alpha}^{(L-1)}$ (this can be seen by symmetrizing $W^{(L-1)}\gives - W^{(L-1)}$).
\end{proof}

\noindent Our next step is to derive a recursion for $B^{(\ell+1)}$ in terms of $B^{(\ell)}$. This is done in Lemma \ref{lem:lemma_three} below, which relies on the following result:
\begin{proposition}\label{prop:relu-4}
Consider a random ReLU network with input dimension $n_0$, $L$ hidden layers of widths $n_1,\ldots, n_L$, and output dimension $n_{L+1}$ as in \eqref{eq:z-def}. Suppose that
\[
\frac{1}{n_1}+\cdots +\frac{1}{n_L}\leq c_1
\]
for some $c_1>0$. For any fixed network input $x_\alpha\in \R^{n_0}$ and any $\ell=1,\ldots, L$ we have
\begin{equation}\label{eq:relu-4}
    \E{\frac{1}{n_\ell}\sum_{j=1}^{n_{\ell}}\lr{z_{j;\alpha}^{(\alpha)}}^4} = \Theta\lr{\frac{1}{n_0^2}\norm{x_\alpha}^4},
    %\E{\frac{1}{n_\ell}\sum_{j=1}^{n_{\ell}}\lr{z_{j;\alpha}^{(\alpha)}}^4} = \Theta\lr{\frac{1}{n_0^2}\norm{x_\alpha}^4},
\end{equation}
where the implicit constants depend on $c_1$ but are otherwise independent are $n$,$\ell$.
\end{proposition}
\begin{proof}
This result is proved in Theorem 1 \cite{hanin2018neural}.
\end{proof}

\noindent We have the following result.
\begin{lemma}\label{lem:lemma_three}
For any depth $\ell\leq L$,  $B^{(\ell)}$   satisfies the following recursion:
\begin{equation}
\begin{aligned}\label{eq:Blrecursionfinal}
   B^{(\ell)} &= 
   \Theta\lr{\frac{(\eta_W^{(\ell)})^2n_{\ell-1}^2}{n_Ln_\ell}\frac{1}{n_0^2}\norm{x_\alpha}^4 }+\frac{\eta_W^{(\ell)}n_{\ell-1}}{n_\ell}C^{(\ell-1)}
   %\Theta\lr{\frac{(\eta_W^{(\ell)})^2n_{\ell-1}^2}{n_Ln_\ell}\frac{1}{n_0^2}\norm{x_\alpha}^4 e^{5\sum_{\ell'=1}^{\ell-2}\frac{1}{n_{\ell'}}}}+\frac{\eta_W^{(\ell)}n_{\ell-1}}{n_\ell}C^{(\ell-1)}\\
   +\frac{1}{n_\ell}\twiddle{B}^{(\ell-1)}+\lr{1+\frac{1}{n_\ell}}B^{(\ell-1)},
\end{aligned}
\end{equation}
where $C^{(\ell)},\twiddle{B}^{(\ell)}>0$ are  defined as follows:
\begin{align}
    \label{E:twiddleB-form}\twiddle{B}^{(\ell)}&:=\frac{1}{n_{\ell+1}}\E{\frac{1}{n_L}\sum_{\mu_1,\mu_2\leq \ell} \eta_{\mu_1}\eta_{\mu_2} \frac{1}{n_{\ell}^2}\sum_{j_1,j_2=1}^{n_{\ell}}\lr{\partial_{\mu_1}z_{j_1;\alpha}^{(\ell)}\partial_{\mu_2}z_{j_2;\alpha}^{(\ell)}}^2},\\
    \label{E:C-form}C^{(\ell)}&:=\E{ \frac{1}{n_L}\sum_{\mu\leq \ell}\eta_{\mu} \frac{1}{n_{\ell}^2}\sum_{j_1,j_2=1}^{n_{\ell}} \lr{z_{j_1;\alpha}^{(\ell)}\partial_{\mu}z_{j_2;\alpha}^{(\ell)}}^2}.%\\
    %\twiddle{C}^{(\ell)}&:=\E{ \frac{1}{n_L}\sum_{\mu\leq \ell}\eta_{\mu} \frac{1}{n_{\ell}}\sum_{j_1,j_2=1}^{n_{\ell}} z_{j_1}^{(\ell)}\partial_{\mu}z_{j_1}^{(\ell)}z_{j_1}^{(\ell)}\partial_{\mu}z_{j_2}^{(\ell)}}.
\end{align}
\end{lemma}

\begin{proof}[Proof of \autoref{lem:lemma_three}]  We distinguish several cases to expand the recursion of $B^{(\ell)}$. If $\mu_1,\mu_2\in \ell$, then the contribution to $B^{(\ell)}$ is
   \begin{align}\label{eq:Brec1}
    \hspace{-.6cm}  \frac{(\eta_W^{(\ell)})^2n_{\ell-1}^2}{n_Ln_\ell}\E{\frac{1}{n_{\ell-1}^2}\sum_{j_1,j_2=1}^{n_\ell-1} \lr{\sigma_{j_1}^{(\ell-1)}\sigma_{j_2}^{(\ell-1)}}^2}=\frac{(\eta_W^{(\ell)})^2n_{\ell-1}^2}{n_Ln_\ell}\Theta\lr{\frac{1}{n_0^2}\norm{x_\alpha}^2} %e^{5\sum_{\ell'=1}^{\ell-2}\frac{1}{n_{\ell'}}}}.
   \end{align}
   Further, if $\mu_1\leq \ell-1$ and $\mu_2\in \ell$ (or vice versa), then the contribution to $B^{(\ell)}$ is
\begin{align}\label{eq:Brec2}
\hspace{-1cm}  2\frac{\eta_W^{(\ell)}n_{\ell-1}}{n_\ell}\E{ \frac{1}{n_L}\sum_{\mu_1\leq \ell-1}\eta_{\mu_1} \frac{1}{n_{\ell-1}}\sum_{k=1}^{n_{\ell-1}} \lr{\sigma_k^{(\ell-1)}}^2 \frac{1}{n_\ell}\sum_{j=1}^{n_\ell}\lr{\partial_{\mu_1}z_j^{(\ell)}}^2} = \frac{\eta_W^{(\ell)}n_{\ell-1}}{n_\ell} C^{(\ell-1)}.
\end{align}
Finally, if $\mu_1,\mu_2\leq \ell-1$, we find the contribution to $B^{(\ell)}$ is
\begin{align}\label{eq:Brec3}
    &\hspace{-1.3cm}\E{\frac{1}{n_L}\sum_{\mu_1,\mu_2\leq \ell-1}\eta_{\mu_1}\eta_{\mu_2}\left\{\frac{1}{n_\ell}\lr{\partial_{\mu_1}z_1^{(\ell)}\partial_{\mu_2}z_1^{(\ell)}}^2+\lr{1-\frac{1}{n_\ell}} \partial_{\mu_1}z_1^{(\ell)}\partial_{\mu_2}z_1^{(\ell)}\partial_{\mu_1}z_2^{(\ell)}\partial_{\mu_2}z_2^{(\ell)}\right\}}\notag\\
    &\hspace{-1.3cm}= \lr{1+\frac{1}{n_\ell}}B^{(\ell-1)}+\frac{1}{n_\ell}\twiddle{B}^{(\ell-1)}.
\end{align}
We adding the contributions \eqref{eq:Brec1}, \eqref{eq:Brec2} and \eqref{eq:Brec3} in \eqref{eq:bl_simplification1} gives the stated result.
\end{proof}

\noindent We now compute the recursion that $\twiddle{B}^{(\ell)}$ satisfies.

\begin{lemma}\label{lem:lemma_four}
For any depth $\ell\leq L$,  $\twiddle{B}^{(\ell)}$ defined in  \eqref{E:twiddleB-form}  satisfies the following recursion:
\begin{equation}\label{eq:tildBrec}
\begin{aligned}
\frac{1}{n_\ell}\twiddle{B}^{(\ell)} &=  \Theta\lr{\frac{(\eta_{W}^{(\ell)})^2n_{\ell-1}^2}{n_Ln_\ell}\frac{\norm{x_\alpha}^4}{n_0^2}}+\frac{\eta_{W}^{(\ell)}n_{\ell-1}}{ n_\ell}C^{(\ell-1)}+\frac{n_{\ell-1}}{n_\ell}\frac{1}{n_{\ell-1}}\twiddle{B}^{(\ell-1)}+\frac{2}{n_\ell^2}B^{(\ell-1)}.
% \Theta\lr{\frac{(\eta_{W}^{(\ell)})^2n_{\ell-1}^2}{n_Ln_\ell}\frac{\norm{x_\alpha}^4}{n_0^2}e^{5\sum_{\ell'=1}^{\ell-2}\frac{1}{n_{\ell'}}}}+\frac{\eta_{W}^{(\ell)}n_{\ell-1}}{ n_\ell}C^{(\ell-1)}+\frac{n_{\ell-1}}{n_\ell}\frac{1}{n_{\ell-1}}\twiddle{B}^{(\ell-1)}\\
%&+\frac{2}{n_\ell^2}B^{(\ell-1)}.
\end{aligned}
\end{equation}
\end{lemma}

\begin{proof}[Proof of \autoref{lem:lemma_four}] We apply the same proof strategy as in \autoref{lem:lemma_three} to get the result.
\end{proof}

\noindent Note that \eqref{eq:Blrecursionfinal} and \eqref{eq:tildBrec} also depends on $C^{(\ell)}$. Its recursion is given by the following lemma. 
\begin{lemma}\label{lem:lemma_six} For any depth $\ell\leq L$, $C^{(\ell)}$ defined in \eqref{E:C-form} satisfies the following recursion
\begin{align}\label{eq:celldef}
     C^{(\ell)} = 
     \Theta\lr{\eta_W^{(\ell)} \frac{n_{\ell-1}}{n_L} \frac{\norm{x_\alpha}^4}{n_0^2}} +\frac{1}{n_\ell}C^{(\ell-1)}  +\lr{1+\frac{1}{n_\ell}} \twiddle{C}^{(\ell-1)},
     %\Theta\lr{\eta_W^{(\ell)} \frac{n_{\ell-1}}{n_L} \frac{\norm{x_\alpha}^4}{n_0^2}e^{5\sum_{\ell'=1}^{\ell-1}\frac{1}{n_\ell'}}} +\frac{1}{n_\ell}C^{(\ell-1)}  +\lr{1+\frac{1}{n_\ell}} \twiddle{C}^{(\ell-1)},
\end{align}
where $\twiddle{C}^{(\ell)}>0$ is a sequence defined as
\begin{align}
   \twiddle{C}^{(\ell)} := \frac{1}{n_L}\E{\sum_{\mu\leq \ell}\eta_\mu \frac{1}{n_{\ell}^2}\sum_{j_1,j_2=1}^{n_{\ell}} \partial_\mu z_{j_1}^{(\ell)}  z_{j_1}^{(\ell)} \partial_\mu z_{j_2}^{(\ell)}  z_{j_2}^{(\ell)}}.
\end{align}
\end{lemma}
\begin{proof}[Proof of \autoref{lem:lemma_six}]  We distinguish several cases to expand the recursion of $C^{(\ell)}$. If $\mu \in\ell$, the contribution to \eqref{E:C-form} is
    \begin{align}\label{eq:eq1CL}
        \eta_W^{(\ell)} \frac{n_{\ell-1}}{n_L} \E{\frac{1}{n_{\ell-1}^2}\sum_{j_1,j_2=1}^{n_{\ell-1}}\lr{z_{j_1}^{(\ell-1)}z_{j_2}^{(\ell-1)}}^2} =\eta_W^{(\ell)} \Theta\lr{\frac{n_{\ell-1}}{n_L}\frac{\norm{x_\alpha}^4}{n_0^2}}%\eta_W^{(\ell)} \Theta\lr{\frac{n_{\ell-1}}{n_L}\frac{\norm{x_\alpha}^4}{n_0^2}e^{5\sum_{\ell'=1}^{\ell-1}\frac{1}{n_{\ell'}}}}.
    \end{align}
Finally, when $\mu \leq \ell-1$, the contribution to \eqref{E:C-form} is
    \begin{equation}
    \begin{aligned}\label{eq:eq2CL}
    &\frac{1}{n_L}\E{\sum_{\mu\leq \ell-1} \eta_{\mu} \left\{\frac{1}{n_\ell} \lr{\partial_\mu z_1^{(\ell)} z_1^{(\ell)}}^2+ \lr{1-\frac{1}{n_\ell}} \lr{\partial_{\mu} z_1^{(\ell)}}^2\lr{z_2^{(\ell)}}^2\right\}}\\
    =& C^{(\ell-1)}+\frac{1}{n_\ell} \twiddle{C}^{(\ell-1)}.
\end{aligned}
\end{equation}
Combining \eqref{eq:eq1CL} and \eqref{eq:eq2CL} yields the result.
\end{proof}
\noindent We finally find the recursion of $\twiddle{C}^{(\ell)}$ that appears in \eqref{eq:celldef}.

\begin{lemma}\label{lem:lemma_seven}
For any depth $\ell\leq L$,  $\twiddle{C}^{(\ell)}$ satisfies the following recursion:
\begin{align}
\twiddle{C}^{(\ell)}=\Theta\lr{\frac{\eta_W^{(\ell)} n_{\ell-1}}{n_\ell n_L} \frac{\norm{x_\alpha}^4}{n_0^2}}+ \frac{1}{n_\ell}C^{(\ell-1)}+\lr{1+\frac{1}{n_\ell}}\twiddle{C}^{(\ell-1)}.%\Theta\lr{\frac{\eta_W^{(\ell)} n_{\ell-1}}{n_\ell n_L} \frac{\norm{x_\alpha}^4}{n_0^2}e^{5\sum_{\ell'=1}^{\ell-1}\frac{1}{n_{\ell'}}}}+ \frac{1}{n_\ell}C^{(\ell-1)}+\lr{1+\frac{1}{n_\ell}}\twiddle{C}^{(\ell-1)}.
\end{align}
\end{lemma}
\begin{proof}[Proof of \autoref{lem:lemma_seven}]
We apply the same proof strategy as in \autoref{lem:lemma_six} to get the result.
\end{proof}

\begin{lemma}\label{lem:CCtwiddle}
For any depth $\ell\leq L$, we have:
\begin{align}
\twiddle{C}^{(\ell)}&=O(n^{-1}),\\ C^{(\ell)} &= %\Theta\lr{\frac{\norm{x_\alpha}^4}{2n_0^2}\sum_{\ell'=1}^{\ell} \frac{\eta_W^{(\ell')}n_{\ell'-1}}{n_L} \exp\left[5\sum_{\ell''=1}^{\ell'}\frac{1}{n_{\ell''}}\right]}
\Theta\lr{\frac{\norm{x_\alpha}^4}{2n_0^2}\sum_{\ell'=1}^{\ell} \frac{\eta_W^{(\ell')}n_{\ell'-1}}{n_L} }\label{eq:penCl}
\end{align}
\end{lemma}
\begin{proof}[Proof of \autoref{lem:CCtwiddle}] The first result is obtained by observing that there is extra $1/n_L$ in front of $\twiddle{C}^{(\ell)} $. Regarding the recursion of $C^{(\ell)}$, we use the fact $\twiddle{C}^{(\ell)} $ is small in \eqref{eq:celldef} and then sum this equation for $\ell'=1,\dots,\ell$ to obtain the value of $C^{(\ell)}$.
\end{proof}

We now specialize to the setting of uniform layer width $n_\ell = n$ and a global learning rate $\eta_\mu = \eta$ to obtain 
\begin{align*}
C^{(\ell)} &= \Theta\lr{\eta\ell}\quad \Longrightarrow\quad 
 \frac{1}{n}\twiddle{B}^{(\ell)} = \Theta\lr{\eta^2\ell^2}\quad \Longrightarrow \quad  B^{(\ell)} = \Theta\lr{\eta^2L^3}\lr{1+O(L^{-1})},
\end{align*}
completing the proof of Theorem \ref{thm:main}.

\section{Conclusion}
In this short note we've computed how variable network depth influences the learning rate predicted by the $\mu$P heurisdtic. We found that, unlike with respect to width, this learning rate has a non-trivial power law scaling with respect to depth (see Theorem \ref{thm:main}). We leave for future work empirical validation of whether this depth dependence indeed leads to learning rate transfer in practice.

\bibliography{references}
\bibliographystyle{plain}
 \end{document}